\documentclass[final,11pt]{article}
\usepackage{ifdraft}
\usepackage{savesym}
\usepackage{amsfonts,amsmath,amssymb,amstext}
\usepackage{mathrsfs}
\savesymbol{iint}
\usepackage[varg]{txfonts} 
\usepackage{fullpage}
\usepackage{bm}            
\usepackage[colorlinks,linkcolor=blue,filecolor=blue,citecolor=blue, 
  urlcolor  = blue]{hyperref}  
\usepackage{xspace}
\usepackage{colortbl}

\definecolor{refkey}{gray}{1}
\definecolor{labelkey}{gray}{1}

\newcommand{\conj}{\mathbf{C}}
\DeclareMathOperator{\RR}{\mathbb R}       
\def\CCC{{\mathcal C}}     
\def\DDD{{\mathcal D}}     
\def\UUU{{\mathcal U}}     %
\DeclareMathOperator{\through}{,\dotsc,}
\DeclareMathOperator{\oo}{\{-1,1\}}    
\def\to{{\rightarrow}}
\def\poly{\mathop{\rm{poly}}\nolimits}

\def\eps{{\epsilon}}        %
\DeclareMathOperator{\E}{\mathbb{E}}

\newcommand{\angles}[1]{\langle #1 \rangle}

\newcommand{\abs}[1]{\lvert#1\rvert}
\newcommand{\Abs}[1]{\left\lvert#1\right\rvert}
\newcommand{\norm}[1]{\lVert#1\rVert}

\newcommand{\Paren}[1]{\left(#1 \right )}
\newcommand{\brac}[1]{[#1 ]}
\newcommand{\Brac}[1]{\left[#1 \right]}
\newcommand{\set}[1]{\{#1\}}

\newcommand{\Stab}{\mathbb S}       

\newcommand{\fhat}{{\hat{f}}}     
\newcommand{\ind}{{\textbf{1}}}    

\newcommand{\etal}{{et al.\ }}
\newcommand{\NB}{{N.B.:\ }}
\newcommand{\ignore}[1]{}
\newcommand{\hnote}[1]{\footnote{{\sf [Homin: {#1}\sf ] }}}

\newtheorem{theorem}{Theorem}
\newtheorem{lemma}[theorem]{Lemma}

\newtheorem{definition}[theorem]{Definition}
\newtheorem{corollary}[theorem]{Corollary}
\newtheorem{proposition}[theorem]{Proposition}

\newenvironment{proof}{\noindent {\em Proof.}}{\medskip}
\def\squareforqed{\hbox{\(\blacksquare\)}}
\def\qed{\ifmmode\squareforqed\else{\unskip\nobreak\hfill\penalty50\hskip1em\null\nobreak\hfil\squareforqed\parfillskip=0pt\finalhyphendemerits=0\endgraf}\fi}

\title{Submodular Functions Are Noise Stable}
\author{Mahdi Cheraghchi\\
UT Austin\\
{\tt mahdi@cs.utexas.edu} 
 \and 
Adam Klivans \\
UT Austin\\
{\tt klivans@cs.utexas.edu} 
\and
Pravesh Kothari\\
UT Austin\\
{\tt kothari@cs.utexas.edu} 
\and 
Homin K. Lee\thanks{Supported by 
  NSF grant 1019343 subaward CIF-B-108}\\
UT Austin\\
{\tt homin@cs.utexas.edu} }
\begin{document}
\maketitle
\ifoptiondraft{\draftbox}

\begin{abstract}
We show that all non-negative submodular functions have high {\em
  noise-stability}. 
As a consequence, we obtain a polynomial-time learning algorithm for
this class with respect to any product distribution on $\{-1,1\}^n$
(for any constant accuracy parameter $\epsilon$).  Our algorithm also
succeeds in the agnostic setting.  Previous work on learning
submodular functions required either query access or strong
assumptions about the types of submodular functions to be learned (and
did not hold in the agnostic setting).
\end{abstract}


\section{Introduction}
A function $f:2^{[n]}\rightarrow \RR$ is \emph{submodular} if
\[\forall S,T \subseteq [n]: f(S \cup T) + f(S \cap T) \leq f(S) + f(T).\]
Submodular functions have been extensively studied in the context of combinatorial
optimization \cite{Edmonds-1971,NWF-1978,FNW-1978,Lovasz-1983} where 
the functions under consideration (such as the cut function of a graph) are submodular.
An equivalent formulation of submodularity is that of decreasing 
marginal returns,
\[ \forall S\subseteq T \subseteq [n], i\in [n]\setminus T: 
f(T \cup \set{i})-f(T) \leq f(S \cup \set{i})-f(S),\]
and thus submodular functions are also a topic of study in economics and the algorithmic game theory community \cite{DNS-2006,MosselRoch-2007}.  In most contexts, the submodular functions considered are non-negative \cite{DNS-2006,FMV-2007,MosselRoch-2007,Vondrak-2009,OveisGharanVondrak-2011,BalcanHarvey-2011,GHRU-2011}, 
and we will be focusing on non-negative submodular functions as well.

The main contribution of this paper is a proof that non-negative submodular
functions are {\em noise stable.}  Informally, a noise stable function
$f$ is one whose value on a random input $x$ does not change much if
$x$ is subjected to a small, random perturbation.  Noise stability is
a fundamental topic in the analysis of Boolean functions with
applications in hardness of approximation, learning theory, social
choice, and pseudorandomness \cite{KKL-1988,Hastad-1997,BKS-1999,ODonnell-2002,KOS-2002,MOO-2005}.

In order to define noise stability, we first define a noise operator that
acts on $\oo^n$.

\begin{definition}[Noise operators]\label{def:noiseop}
For any product distribution $\Pi=\Pi_1\times \Pi_2 \times \cdots \times \Pi_n$ over $\oo^n$,
$\rho\in [0,1]$, $x\in \oo^n$, let the random variable $y$ drawn from the
distribution $N_\rho(x)$ over $\oo^n$
have $y_i=x_i$ with probability $\rho$ and be randomly drawn from $\Pi_i$
with  probability $1-\rho$. 
The \emph{noise operator} $T_\rho$ on $f:\oo^n\rightarrow \RR$ is defined by letting
$T_\rho f:\oo^n\rightarrow \RR$ be the function given by $T_\rho f(x)=\E_{y\sim N_\rho(x)} f(y)$.
\end{definition}
\noindent\NB For the uniform distribution $y\sim N_\rho(x)$ has
 $y_i=x_i$ with probability $1/2 + \rho/2$, and $y_i=-x_i$ with probability
$1/2-\rho/2$.

Now we can precisely define noise stability:

\begin{definition}[Noise stability] The \emph{noise stability} of $f$ at noise rate $\rho$ is defined to be
\[\Stab_\rho(f)=\angles{f,T_\rho f} = \E_{x\sim\Pi}[f(x) T_\rho f (x)].\]
\end{definition}

The precise statement of our main theorem is as follows (see Section~\ref{sec:prelim} for definitions):
\begin{theorem}\label{thm:stability}
 Let $\Pi=\Pi_1\times \Pi_2 \times \cdots \times \Pi_n$ 
be a product distribution over $\oo^n$  
with  minimum probability $p_{min}$ 
and let $f:\oo^n\rightarrow \RR^{+}$ be a submodular function. 
Then for all, $\rho\in [0,1]$, 
\[\Stab_\rho(f)\geq (2\rho - 1 + 2p_{min}( 1 - \rho)) \norm{f}_2^2.\]
\end{theorem}
\noindent\NB For the uniform distribution we get the bound
$\Stab_\rho(f)\geq \rho\norm{f}_2^2.$

Given the high noise-stability of submodular functions, we can apply
known results from 
computational learning theory to show that submodular functions are well-approximated
by low-degree polynomials and can be learned agnostically.   Our main
learning result is as follows:

\begin{corollary}\label{cor:fullpower}
Let $\CCC$ be the class of non-negative submodular functions with $\norm{f}_2=1$ 
and let $\DDD$ be any distribution on $\oo^n\times \RR$ such that
the marginal distribution over $\oo^n$ is a product distribution.
Then there is a statistical query algorithm that outputs a hypothesis $h$ 
with probability $1-\delta$
such that
\[\E_{(x,y)\sim \DDD}[|h(x) - y|] \leq opt +\eps,  \]
in time
$\poly(n^{O(1/\eps^2)},\log(1/\delta))$,
\end{corollary}
Here $opt$ is the $L_1$-error of the best fitting concept in the concept class.
(See Section~\ref{sec:learn} for the precise definition.)
Note that the above algorithm will succeed given only statistical
 query access \cite{Kearns-1993} to the underlying function to be learned.
It can be shown that the $L_2$-norm of a submodular function is always
within a constant factor of its mean squared.
Thus, the algorithm can estimate the $L_2$-norm of the submodular
function $f$ to very high accuracy using Chernoff-Hoeffding bounds
and scale the function by its mean so that its $L_2$-norm is 1.

\subsection{Related Work}
Recently the study of learning submodular functions was initiated in two
very different contexts. Gupta \etal \cite{GHRU-2011} gave an algorithm for learning 
bounded submodular functions that arose as a technical necessity for differentially privately releasing 
the class of disjunctions. Their learning algorithm requires value query access to the target function, but
their algorithm works even when the value queries are answered with additive error (value queries
that are answered with additive error at most $\tau$ are said to be $\tau$-tolerant). 
\begin{theorem}[\cite{GHRU-2011}]\label{thm:ghru}
Let $\eps,\delta>0$ and let $\Pi$ be any product distribution over $[n]$. There is a learning algorithm that  
when given ($\eps/4$-tolerant) value query access to any submodular function $f:2^{[n]}\rightarrow [0,1]$, 
outputs a hypothesis $h$ in time $n^{O(\log(1/\delta)/\eps^2)}$ such that,
\[\Pr_{S\sim \Pi}\Brac{\abs{f(S)-h(S)} \leq \eps}\geq 1-\delta.\]
\end{theorem}
The learning algorithm of Gupta \etal crucially relies on its query access to the submodular function
in order to break the function down into Lipschitz continuous parts that are easier to learn.
Compare this to Corollary~\ref{cor:fullpower}, which has similar learning guarantees, but where the learner
only has access to statistical queries and can learn
in the agnostic model of learning. (See Section~\ref{sec:learn}.)

The other recent work \cite{BalcanHarvey-2011} on learning submodular functions
was motivated by bundle pricing and used passive supervised learning as
a model for learning consumer valuations of added options.
In particular, they have a polynomial-time algorithm that can learn
(using random examples only) {\em
  monotone}, non-negative, submodular
functions within a {\em multiplicative} factor of $\sqrt{n}$ over
\emph{arbitrary} distributions.  
As our machinery breaks down over non-product distributions, none of our results hold
in this setting.
For product distributions, Balcan and Harvey
gave the first $\poly(n,1/\eps)$-time algorithm that can learn
(using random examples only) {\em
  monotone}, non-negative, {\em Lipschitz} submodular
functions with minimum value $m$ within a {\em multiplicative} factor of $O(\log(1/\eps)/m)$.

\ignore{
with a better guarantee
for learning 
than $\sqrt{n}$
that is more
comparable to ours.
\begin{theorem}[\cite{BalcanHarvey-2011}]\hnote{is this statement correct? can someone check that it is, and that
our algorithm is actually superior?}
Let $\alpha>0$ and let $\Pi$ be any product distribution over $[n]$. 
There is a learning algorithm that  
when given access to random examples drawn from $\Pi$ labeled by any 
monotone, $(1/n)$-Lipschitz,
submodular 
function $f:2^{[n]}\rightarrow [0,1]$
with minimum value $m$, 
outputs a hypothesis $h$ in time $n^{O(\log(1/\beta)/\alpha^2)}$ such that,
\[\Pr_{S\sim \DDD}\Brac{h(S) \leq f(S) \leq O\Paren{\frac{\log(1/\eps)}{m}} h(S)} \geq 1-\eps.\]
 \end{theorem}
}

\subsection{Applications to Differential Privacy}
We discuss some applications to differential privacy in
Section~\ref{sec:private}.  In particular, we obtain a simple proof of Gupta \etal \cite{GHRU-2011}'s 
recent result on releasing disjunctions with improved parameters. 

\section{Preliminaries}
\label{sec:prelim}

Throughout, we will identify sets
$S\subseteq [n]$ with their indicator vectors $\ind(S)\in\oo^n$ where $\ind(S)_i=1$ 
if $i\in S$ and $\ind(S)_i=-1$ if $i\not\in S$ 
(as opposed to the usual $(0,1)$-indicator vectors).
For any distribution over $\Pi$ over $\oo^n$, 
we define the inner product on functions
$f,g:\oo^n\rightarrow \RR$ by $\angles{f,g}=\E_{x\sim \Pi}[f(x)g(x)]$
and the $L_2$-norm of a function of $f$ as $\norm{f}_2=\sqrt{\angles{f,f}}=\sqrt{\E_{x\sim \Pi} [f(x)^2]}$.

\begin{definition}[Minimum probability] 
Let $\Pi=\Pi_1\times \Pi_2 \times \cdots \times \Pi_n$
be a product distribution over
$\oo^n$, and let
$p_i:=\Pr_{x_i\sim \Pi_i}[x_i=1]$, then $p_{min} = \min_{i\in[n]} \{p_i\}$.
\end{definition}

\section{Submodular Functions are Noise Stable}
We will start by showing that submodular functions are noise stable under 
the uniform distribution as a warm-up as the notation is less cumbersome in this setting.
In Section~\ref{subsec:prod} we will prove Theorem~\ref{thm:stability} in the
general setting
of arbitrary product distributions.

\subsection{Uniform Distribution}\label{subsec:unif}
For the rest of Section~\ref{subsec:unif} we will assume that the distribution over inputs is
uniform.

Let the Fourier expansion of $f$ be given by $\sum_{S\subseteq [n]} \fhat(S)\chi_S$, it can be shown 
that 
\[T_\rho f(x)=\sum_{S\subseteq [n]} \rho^{|S|}\fhat(S)\chi_S(x), 
\hspace{1cm} \textrm{ and thus }\hspace{1cm}
\Stab_\rho(f) = \sum_{S\subseteq [n]} \rho^{|S|}\fhat(S)^2.\]

The following lemma is our key observation.
\begin{lemma}\label{lem:thresh}
 Let $f:\oo^n\rightarrow \RR$ be a submodular function.
Then for all $x\in \oo^n$, $\rho\in [0,1]$, 
\[T_\rho f(x) \geq \rho f(x)
+ ((1-\rho)/2)(f(-1^n)+f(1^n)).\]
For  $f:\oo^n\rightarrow \RR^+$,
$T_\rho f(x) \geq \rho f(x)$.
\end{lemma}
\begin{proof}
We will be viewing
the domain of $f$ as $2^{[n]}$, and the input $x\in \oo^n$ as 
$X\in 2^{[n]}$ such that $\ind(X)=x$.
For a fixed $x\in\oo^n$, let $\pi:[n]\rightarrow [n]$ be a permutation such that
$x_{\pi(1)} \geq \cdots \geq x_{\pi(n)}$, and then define 
$X_j=\set{\pi(1)\through \pi(j)}$. (\NB $X_0=\emptyset$ and $X_n=[n]$.)
Finally, we define $x_{\pi(0)}=1$ and $x_{\pi(n+1)}=-1$.
Note that there is only one value $j\in \set{0\through n}$ for which $x_{\pi(j)}\not=x_{\pi(j+1)}$.

\begin{eqnarray*}
\E_{Y\sim {N_\rho(X)}} f(Y) & = & 
f(X_0) + \E_{Y\sim {N_\rho(X)}} \sum_{j=1}^n f(Y \cap X_j)
- f(Y \cap X_{j-1})\\
& \geq & f(X_0) + \E_{Y\sim N_\rho(X)} \sum_{j=1}^n f((Y \cap \set{\pi(j)}) \cup X_{j-1})
- f(X_{j-1})\\
& = & f(X_0) + \sum_{j=1}^n \frac{1+\rho x_{\pi(j)}}{2}(f(X_{j})-f(X_{j-1}))\\
& = & \sum_{j=0}^n \frac{\rho}{2}( x_{\pi(j)}- x_{\pi(j+1)})f(X_{j})
+\frac{1-\rho}{2}(f(X_0)+f(X_n))\\
& = & \rho f(X)+\frac{1-\rho}{2}(f(X_0)+f(X_n)).
\end{eqnarray*}
The inequality is due to the decreasing marginal returns 
characterization of the
 submodularity of $f$. The equality after that comes
from moving the expectation inside and observing that each summand is non-zero only if $\pi(j)\in Y$. This
happens with probability $(1+\rho)/2$ when $x_{\pi(j)}=1$
and with probability $(1-\rho)/2$ when $x_{\pi(j)}=-1$.
\end{proof}

\noindent{\em Remark.} The proof technique is not new (for
instance it was used by Madiman and Tetali \cite{MadimanTetali-2010} to show a large class of
Shannon-type inequalities for the joint entropy function). In fact, it can be viewed
as a special case of the ``Threshold Lemma'' \cite{Vondrak-2009}. However, to 
the best of our knowledge the statement of Lemma~\ref{lem:thresh} has never been
expressed using the language of noise operators.

\begin{corollary}
Let $f:\oo^n\rightarrow \RR^+$ be a submodular function.
Then for all $\rho\in [0,1]$, 
\[\Stab_\rho(f)\geq \rho \norm{f}_2^2.\]
\end{corollary}

\subsection{Product Distributions}\label{subsec:prod}

For the rest of Section~\ref{subsec:prod} we will assume that the distribution is a 
product distribution $\Pi = \Pi_1 \times \Pi_2 \times \cdots \times 
\Pi_n$ on $\oo^n$ with minimum probability $p_{min}$. 

\begin{lemma}\label{lem:genthresh}
Let $\Pi=\Pi_1\times \Pi_2 \times \cdots \times \Pi_n$ 
be a product distribution over $\oo^n$  
with  minimum probability $p_{min}$ 
and let $f:\oo^n\rightarrow \RR^{+}$ be a submodular function. 
Then for all $x\in \oo^n$, $\rho\in [0,1]$, 
\[T_\rho f(x) \geq( (2\rho - 1) + 2p_{min}(1 - \rho)) f(x)\]
\end{lemma}
\begin{proof}
As in the proof of Lemma~\ref{lem:thresh}, we will be viewing
the domain of $f$ as $2^{[n]}$, and the input $x\in \oo^n$ as 
$X\in 2^{[n]}$ such that $\ind(X)=x$.
For a fixed $x\in\oo^n$, let $\pi:[n]\rightarrow [n]$ be a permutation 
such that
$x_{\pi(1)} \geq \cdots \geq x_{\pi(n)}$, and then define 
$X_j=\set{\pi(1)\through \pi(j)}$. (\NB $X_0=\emptyset$ and $X_n=[n]$.)
\begin{eqnarray*}
\E_{Y\sim {N_{\rho}(X)}} f(Y) & =  &
f(X_0) + \E_{Y\sim {N_{\rho}(X)}} \sum_{j=1}^n f(Y \cap X_j)
- f(Y \cap X_{j-1})\\
& \geq&  f(X_0) + \E_{Y\sim {N_{\rho}(X)}}  \sum_{j=1}^n f((Y \cap \set{\pi(j)})  \cup X_{j-1})
- f(X_{j-1})\\
& = & f(X_0) + \sum_{j=1}^n  \Brac{\frac{1}{2} + \frac{1}{2} x_{\pi(j)}  - x_{\pi(j)}(1 -\rho)(1 - p_{\pi(j)})  }     (f(X_{j})-f(X_{j-1}))\\
& =&  \sum_{j=1}^{n-1}  \Brac{(1- \rho)\Paren{x_{\pi(j)}p_{\pi(j)} - x_{\pi(j+1)}p_{\pi(j+1)}} - \Paren{\frac{1}{2} - \rho}\Paren{x_{\pi(j)} - x_{\pi(j+1)}}  } f(X_{j})\\
&&+ \Paren{\frac{1}{2} - \frac{x_{\pi(1)}}{2} + x_{\pi(1)}( 1 - \rho)( 1 - p_{\pi(1)}) } f( X_0) \\
& &+ \Paren{\frac{1}{2} +  \frac{x_{\pi(n)}}{2} - x_{\pi(n)} ( 1 - \rho)( 1 - p_{\pi(n)}) } f(X_n) \\
& \geq &(2\rho - 1 + 2p_{min}( 1 - \rho)) f(X).
\end{eqnarray*}
The first inequality comes from using submodularity in each term of the summation.
The equality after that comes from moving the expectation inside and observing that each summand is non-zero only if $\pi(j)\in Y$.
This
happens with probability $\rho+(1-\rho)p_{\pi(j)}$ when $x_{\pi(j)}=1$
and with probability $(1-\rho)(1-p_{\pi(j)})$ when $x_{\pi(j)}=-1$.
Finally, the last line follows by the non-negativity of $f$ and observing 
that for any values of $x_{\pi(1)}$ and $x_{\pi(n)}$, the coefficients of 
$f(\emptyset)$ and $f([n])$ are 
non-negative and the coefficient of $f(X)$ is 
at least $(2\rho - 1 + 2p_{min}( 1 - \rho))$.
\end{proof}

As with Fourier analysis over the uniform distribution, it can be 
easily verified that 
 \[\Stab_\rho(f) =\angles{f,T_\rho f}= \sum_{S \subseteq [n]} \rho^{|S|} \hat{f}(S)^{2} \] over any 
product distribution $\Pi$,
where the Fourier coefficients are now defined with respect to the Gram-Schmidt
orthonormalization of the $\chi$ basis with respect to the the $\Pi$-norm \cite{Bahadur-1961,FJS-1991}. 
Thus, once again we get a lower-bound on the noise-stability of submodular
functions as an immediate consequence of Lemma~\ref{lem:genthresh}.

\begin{theorem}[Theorem~\ref{thm:stability} Restated]
 Let $\Pi=\Pi_1\times \Pi_2 \times \cdots \times \Pi_n$ 
be a product distribution over $\oo^n$  
with  minimum probability $p_{min}$ 
and let $f:\oo^n\rightarrow \RR^{+}$ be a submodular function. 
Then for all, $\rho\in [0,1]$, 
\[\Stab_\rho(f)\geq (2\rho - 1 + 2p_{min}( 1 - \rho)) \norm{f}_2^2.\]
\end{theorem}

\section{Learning}\label{sec:learn}

In the agnostic learning framework \cite{KSS-1992}, the learner receives labelled examples
$(x,y)$ drawn from a fixed distribution over example-label pairs.

\begin{definition}[Agnostic Learning]
Let $\DDD$ be any distribution on $\oo^n\times \RR$ such that
the marginal distribution over $\oo^n$ is a product distribution $\Pi$. Define
\[opt=\min_{f\in \CCC} \E_{(x,y)\sim \DDD}\Brac{\abs{f(x)-y}}.\] That is, $opt$ is the
error of the best fitting $L_1$-approximation in $\CCC$ with respect to $\DDD$.

We say that an algorithm $A$ agnostically learns a concept class $\CCC$ over
$\Pi$ if the following holds for any $\DDD$ with marginal $\Pi$:
if $A$ is given random examples drawn from $\DDD$, then with high probability 
$A$ outputs a hypothesis $h$ such that $\E_{(x,y)\sim\DDD }\Brac{\abs{h(x)-y}}\leq opt+\eps$.
\end{definition}

 The following lemma, considered folklore (see \cite{KOS-2002}), shows that
noise stable functions are well-approximated by low-degree polynomials.

\begin{lemma}
 Let $\Pi=\Pi_1\times \Pi_2 \times \cdots \times \Pi_n$ 
be a product distribution over $\oo^n$, and 
let $f:\oo^n\rightarrow \RR$ be a function such that $\norm{f}_2=1$ and $\Stab_{\rho}(f)\geq 1-2\gamma$.
Then there exists a multilinear polynomial $p:\oo^n\rightarrow\RR$ of degree $2/(1-\rho)$
such that
\[\E_{x\sim \Pi} \Brac{(f-p)^2} < \Paren{\frac{2}{1-e^{-2}}} \gamma.\]
\end{lemma}

\ignore{The ``Low-Degree Algorithm'' \cite{LMN-1989,Mansour-1994} shows that by estimating
all the low-degree Fourier coefficients of a function with good Fourier concentration
one can learn the function to high accuracy.
The Low-Degree Algorithm is known to also work under any constant-bounded 
product distributions \cite{FJS-1991} as well.
\begin{corollary}\label{cor:low-degree}
Let $\CCC$ be the class of non-negative submodular functions with $\norm{f}_2=1$ 
and let the distribution over inputs be any product distribution with minimum
probability bounded by $p_{min}$. 
Then the Low-Degree Algorithm
outputs a hypothesis $h$ such that
$\E[(f-h)^2]=O(\gamma)$ given random examples in time
$\poly(n^{2/(1-c)},1/\gamma)$ where $c=(2\rho - 1 + 2p_{min}( 1 - \rho))$.
\end{corollary}
}

The ``$L_1$ Polynomial Regression Algorithm'' due to Kalai \etal \cite{KKMS-2005} 
shows that one can \emph{agnostically} learn low-degree polynomials. 
\begin{theorem}[\cite{KKMS-2005}]
Suppose $\E_{x\sim \DDD_X} \brac{(f-p)^2} < \eps^2$ for some degree $d$ polynomial $p$,
some distribution $\DDD$ on $X\times \RR$ where the marginal $\DDD_X$ is a product distribution on $\oo^n$,
and any $f$ in the concept
class $\CCC$. Then, with probability $1-\delta$, the $L_1$ Polynomial Regression Algorithm
outputs a hypothesis $h$ such that $\E_{(x,y)\sim\DDD }\Brac{\abs{h(x)-y}}\leq opt+\eps$
in time $\poly(n^d/\eps,\log(1/\delta))$.
\end{theorem}

\begin{corollary}
Let $\CCC$ be the class of non-negative submodular functions with $\norm{f}_2=1$ 
and let $\DDD$ be any distribution on $\oo^n\times \RR$ such that
the marginal distribution over $\oo^n$ is a product distribution.
Then for all $f\in \CCC$, 
the $L_1$ Polynomial Regression Algorithm outputs a hypothesis $h$ 
with probability $1-\delta$
such that
\[\E_{(x,y)\sim \DDD}\Brac{\abs{h(x) - y}} \leq opt +\eps,  \]
given random examples in time
$\poly(n^{O(1/\eps^2)}/\eps,\log(1/\delta))$ .
\end{corollary}

We note that the
$L_1$ Polynomial Regression Algorithm can be implemented as a statistical query algorithm \cite{Kalai-2011}. (\NB The access offered to the learning algorithm by the statistical query model is much weaker than that offered by random examples or the tolerant value query model. The tolerant value query model allows arbitrary value queries that get answered with some noise, whereas the statistical query model requires that the queries to be of the form $g:\oo^n \times \RR\rightarrow \RR$ where $g$ is computable by a $\poly(n,1/\eps)$-size circuit, and the answer is $\E_{(x,y)\sim \DDD}[g(x,y)]$ with some noise.)

\begin{corollary}[Corollary~\ref{cor:fullpower} Restated]
Let $\CCC$ be the class of non-negative submodular functions with $\norm{f}_2=1$ 
and let $\DDD$ be any distribution on $\oo^n\times \RR$ such that
the marginal distribution over $\oo^n$ is a product distribution.
Then for all $f\in \CCC$, 
there is a statistical query algorithm that outputs a hypothesis $h$ 
with probability $1-\delta$
such that
\[\E_{(x,y)\sim \DDD}[|h(x) - y|] \leq opt +\eps,  \]
in time
$\poly(n^{O(1/\eps^2)},\log(1/\delta))$.
\end{corollary}

\section{Private Query Release and Low-Degree Polynomials}\label{sec:private}

In this section, we make a simple observation connecting approximability by 
low-degree polynomials with private query release.

In the context of differential privacy, we will call $D\subset X$ a 
\emph{database} and two databases $D,D'\subset X$ are \emph{adjacent}
if one can be obtained from the other by adding a single item.

\begin{definition}[Differential privacy \cite{DMNS-2006}]
An algorithm $A:X^*\rightarrow R$ is \emph{$\eps$-differentially private} if for all 
$Q\subset R$ and every pair of adjacent databases $D,D'$, we have $\Pr[A(D)\in Q]
\leq e^\eps\Pr[A(D') \in Q]$.
\end{definition}

A \emph{counting query} over a database $D$ is just the average value of a query
over each entry in the database.
\begin{definition}[Counting Query Function]
Let  $c:X\rightarrow \RR$ be a real-valued query function. For a fixed $r\in X$,
let $\mathbf{q}_r(c):=c(r)$.
For a class of queries $\CCC$ and a fixed database $D\subset X$,
the \emph{counting query function} $\mathbf{CQ}_D:\CCC\rightarrow \RR$ is the 
function defined by $\mathbf{CQ}_D(c):=\frac{1}{n}\sum_{r\in D}\mathbf{q}_r(c)=\frac{1}{n}\sum_{r\in D}c(r)$.
\end{definition}

A \emph{counting query releasing} algorithm's objective is to release
a data structure $H$ whose answers on queries $c\in \CCC$ are close to
those of the counting query over the original database $D$.

\begin{definition}[Counting query release \cite{GHRU-2011}]
Let $\CCC$ be a class of queries $c$ from $X\rightarrow \RR$, 
and let $\Pi$ be a distribution on $\CCC$. We say that an algorithm $A$ $(\alpha,\beta)$-releases
$\CCC$ over a database $D$ of size $n$, if for $H=A(D)$,
\[\Pr_{c\sim \Pi}\Brac{\Abs{\mathbf{CQ}_D(c)-H(c)}\leq \alpha} \geq 1-\beta.\]
\end{definition}

The following proposition is implicit in \cite{GHRU-2011} using results of
\cite{BDMN-2005} and \cite{KLNRS-2008}.
\begin{proposition}\label{prop:ghru}
For a given concept class $\CCC$ with distribution $\Pi$, if there is a query learning algorithm for
the concept class $\set{\mathbf{CQ}_D  : D\subset X }$
using $q$ $\tau$-tolerant value queries
that outputs a hypothesis $H$ s.t. 
$\Pr_{c\in \Pi}[\abs{\mathbf{CQ}_D(c)-H(c)}\leq \alpha]\geq 1-\beta$, then there is an 
$\eps$-differentially private algorithm that $(\alpha,\beta)$-releases $\CCC$ for
any database of size $|D|\geq q(\log q + \log(1/\delta))/\eps\tau$.
\end{proposition}
For instance, Gupta \etal \cite{GHRU-2011} show that for $\CCC$, the
class of disjunctions, the class $\set{\mathbf{CQ}_D : D\subset X }$ is
a submodular function. Thus, their tolerant value query learning
algorithm for submodular functions leads to a private counting query
release algorithm.

We make the following observation. For a given concept class $\CCC$
with distribution $\Pi$, if for every $r\in X$, $\mathbf{q}_r$ is
well-approximated by a low-degree polynomial with respect to $\Pi$,
then $\mathbf{CQ}_D$ is also well-approximated by a low-degree
polynomial with respect to $\Pi$.  As statistical queries are strictly
weaker than tolerant value queries, the $L_1$ Polynomial Regression
Algorithm satisfies the requirements of Proposition~\ref{prop:ghru},
and we have a private counting query release algorithm for $\CCC$.  We
note that it is easy to see that a $O(\log(1/\alpha))$-degree
polynomial can $L_1$-approximate $\mathbf{q}_r$ to within $\alpha$, when $\CCC$ is the
class of disjunctions, and $\Pi$ is the uniform distribution. 
(If $|r|=O(\log(1/\alpha))$, a $O(\log(1/\alpha))$-degree
polynomial can interpolate the function exactly. Otherwise, 
the constant 1 function is within $\alpha$ of  $\mathbf{q}_r$.)
 Thus,
we are able to retrieve the result of Gupta \etal \cite{GHRU-2011} on
releasing disjunctions easily with an improved running-time of
$|X|^{O(\log(1/\alpha))}$ as opposed $|X|^{O(1/\alpha^2)}$.

\section*{Acknowledgements}
We would like to thank Aaron Roth for explaining \cite{GHRU-2011} to us. 

\newpage
\newcommand{\etalchar}[1]{$^{#1}$}


\end{document}